\documentclass[journal,onecolumn]{IEEEtran}
\usepackage[T1]{fontenc}
\usepackage{microtype,mathtools,amssymb,amsthm,booktabs,array,enumitem}
\usepackage[numbers,sort&compress]{natbib}
\usepackage[ruled,vlined]{algorithm2e}
\usepackage{xcolor,hyperref}
\title{Parameter-Free Interval-Dynamic Regret under Heavy-Tailed Noise}
\author{Vaneet Aggarwal%
\thanks{Vaneet Aggarwal is with Purdue University, West Lafayette, IN, USA.}}
\newcommand{\E}{\mathbb E}
\newcommand{\R}{\mathbb R}
\newcommand{\X}{\mathcal X}
\newcommand{\F}{\mathcal F}
\newcommand{\DD}{\mathcal D}

\newcommand{\Reg}{\operatorname{Regret}}

\newcommand{\ind}{\mathbf 1}

\newcommand{\curv}{\mathfrak h}
\newcommand{\defect}{\mathfrak d}
\newcommand{\clip}{\operatorname{clip}}

\newcommand{\EE}{\mathcal E}

\DeclareMathOperator{\KL}{KL}
\newtheorem{theorem}{Theorem}
\newtheorem{lemma}[theorem]{Lemma}
\newtheorem{proposition}[theorem]{Proposition}
\newtheorem{corollary}[theorem]{Corollary}
\numberwithin{equation}{section}
\numberwithin{theorem}{section}
\setlist[enumerate]{leftmargin=*,itemsep=3pt,topsep=5pt}
\SetAlgoLined\LinesNumbered\DontPrintSemicolon
\SetKwInput{KwInput}{Input}
\SetAlFnt{\small}\SetAlCapFnt{\small\bfseries}
\hypersetup{hidelinks,pdftitle={Parameter-Free Interval-Dynamic Regret under Heavy-Tailed Noise},pdfauthor={Vaneet Aggarwal}}
\renewcommand{\paragraph}[1]{\par\medskip\noindent\textbf{#1}\enspace}

\begin{document}
\maketitle
\begin{abstract}
We study online convex optimization with one unbiased stochastic subgradient per round and an unknown finite conditional $p$th noise moment, $1<p\le2$. For every fixed interval $I$ of length $n$ and comparator path with $\Lambda_I=1+P_I/D$, one learner achieves
\[
\E\Reg_I(u)\le\min\!\left\{GDn,\ C\!\left[GD\sqrt{n(\Lambda_I+\log^2(2T))}
+\sigma Dn^{1/p}(\Lambda_I+\log^2(2T))^{(p-1)/p}\right]\right\}.
\]
The learner uses none of $G,\sigma,p,I,P_I$, and the constant is universal. Interval adaptation adds to comparator complexity, preserving the distinct mean-gradient and noise exponents. The analysis controls calibration in expectation and limits the cost of observation-scale changes. Its general theorem compares to distributions over predictably available experts with relative-entropy dependence on a nonuniform prior. A common prior favors long windows and long restart lengths. With the statistics supplied, the interval cost becomes $1+\log(T/n)$, including the optimal full-horizon static rate. A change-of-measure lower bound identifies the noise power of this logarithm for learners retaining a full-horizon optimal guarantee, under explicit conditions. Static comparisons and deterministic partitions follow from the same decisions.
\end{abstract}
%\begin{IEEEkeywords}
%Online convex optimization, interval-dynamic regret, heavy-tailed noise, sleeping experts
%\end{IEEEkeywords}

\section{Introduction}\label{sec:intro}
An online learner may run continuously while its performance is evaluated on time windows chosen later. A full-horizon guarantee can obscure a short period of poor tracking: regret accumulated before that period may compensate for it. Interval-dynamic regret instead compares the learner with a moving decision sequence on each fixed evaluation window. This distinction is useful when comparator movement is concentrated in an episode or when several windows describe different operating conditions. With stochastic first-order feedback, an additional difficulty is that a single observation can be arbitrarily large and the noise can have infinite variance.

Let $D$ be the domain diameter, $G$ a bound on the selected mean subgradient, and $\sigma^p$ a uniform conditional bound on the central $p$th noise moment. For a fixed interval $I$ and comparator $u$, write
\begin{align}\label{eq:scales}
 n&=|I|, & \Lambda_I&=1+P_I/D,\notag\\
 \rho&=(p-1)/p, & L&=\log(2T).
\end{align}
To put static, dynamic, and interval guarantees in the same notation, define
\begin{equation}\label{eq:benchmark}
 \mathcal B_I(q)=GD\sqrt{nq}+\sigma Dn^{1/p}q^\rho\quad(q\ge1),
 \qquad C_I=GDn.
\end{equation}
Thus $\mathcal B_I(1)$ is the static scale and $\mathcal B_I(\Lambda_I)$ is the dynamic scale on a prescribed window. The exact-gradient dynamic rate is attained by Ader \citep{zhang2018}; the heavy-tailed static rate by AdaGrad \citep{liu2026}; and the joint full-horizon rate by \citet{aggarwal2026}. The question here is the additional cost of supporting unannounced evaluation intervals under unknown noise statistics.

\paragraph{Main result and its significance.}
Theorem~\ref{thm:main} gives
\begin{equation}\label{eq:intro-main}
 \E\Reg_I(u)\le\min\!\left\{C_I,\ C\left[\mathcal B_I(\Lambda_I)+
 GD\sqrt n\,L+\sigma Dn^{1/p}L^{2\rho}\right]\right\}.
\end{equation}
Equivalently, the upper bound is $O(\mathcal B_I(\Lambda_I+L^2))$, capped by $C_I$, with universal constants. Interval adaptation enters additively in comparator complexity. In particular, $\Lambda_I\ge L^2$ suffices to absorb it into the prescribed-window dynamic rate. At $p=3/2$, the noise adaptation factor is $L^{2/3}$ while the mean-gradient factor is $L$. A constant comparator gives the interval-static bound $O(\mathcal B_I(1+L^2))$ on the same decisions. Deterministic partitions give finer comparisons when movement is concentrated in time.

With $G,\sigma,p$ supplied, Theorem~\ref{thm:known} replaces $L^2$ by $\chi_T(n)=1+\log(T/n)$. This preserves the full-horizon rate because $\chi_T(T)=1$. Theorem~\ref{thm:testing} and Corollary~\ref{cor:statistical-price} show that the noise cost $\sigma Dn^{1/p}\log(T/n)^\rho$ is necessary in an explicit regime for learners retaining the optimal full-horizon static rate. Together, these results separate the statistical cost of local evaluation from the additional upper-bound cost of learning the noise statistics.

\paragraph{Technical novelty: expected calibration under finite moments.}
The master uses the bounded multiplier $F(z)=1+\clip(z,-1/2,1/2)$, building on multi-rate multiplicative aggregation \citep{gaillard2014}. Its positive tangent excess and logarithmic defect are both bounded by $\min\{2z^2,|z|\}$. Their conditional expectations are therefore controlled by $4\eta^2(DG)^2+4\eta^p(D\sigma)^p$. The calibration allocates aggregate positive drift into bounded responses. A fixed-threshold exponential test bounds the expected maximum penalty excess by a constant times its conditional-moment bound, including when that bound lies below the initial penalty floor. This control permits a regret calculation directly in expectation while retaining the finite-$p$ term before optimizing the learning rate. The raw subgradient continues to update every base trajectory unchanged.

\paragraph{Technical novelty: local extraction with the correct prior cost.}
An observation larger than $T^3b$ sets the next scale to $2\|g_t\|/T^2$, so successive scales increase by more than $2T$. Only a constant number of phases contribute to the localized comparison outside an explicitly controlled exceptional event. A pathwise counter budget bounds each phase's log-potential by $O(L^2)$, including its trigger observation. Keeping the comparison distribution inside logarithmic extraction yields $\KL(\nu\|\pi)$ in the master theorem. Relative-entropy comparisons are standard in expert aggregation \citep{luo2015}; here they are combined with predictable availability, finite-$p$ defects, and expected calibration. A two-level polynomial prior then separates the cost of locating a window from the costs of selecting its duration and restart length. This gives an explicit prior-sensitive bound as well as the simpler interval rate.

\paragraph{Contributions.}
\begin{enumerate}
\item \textbf{Interval tracking under unknown heavy-tailed noise.} Algorithm~\ref{alg:master} and Theorem~\ref{thm:main} establish \eqref{eq:intro-main} with one shared gradient and $O(\log T)$ projections per round. Corollaries~\ref{cor:static} and~\ref{cor:partition} give static and deterministic-partition comparisons. Table~\ref{tab:comparison} places the results alongside the closest full-horizon and interval guarantees.
\item \textbf{A prior-sensitive sleeping-master theorem.} Theorem~\ref{thm:master} compares to fixed distributions over predictable expert predictions and availability, with relative-entropy complexity and an almost-sure weighted activity budget. Lemmas~\ref{lem:potential} and~\ref{lem:overshoot} prove the calibration bounds; Appendix~\ref{app:master} supplies the complete local comparison. Proposition~\ref{prop:prior} quantifies the nonuniform prior used by both algorithms.
\item \textbf{Upper and lower bounds for the statistical interval cost.} Theorem~\ref{thm:known} gives the supplied-statistics guarantee with cost $1+\log(T/n)$ and an explicit refinement retaining the prior. Theorem~\ref{thm:testing} and Corollary~\ref{cor:statistical-price} establish the corresponding logarithmic noise obstruction under a full-horizon performance requirement.
\end{enumerate}

\section{Related Work}\label{sec:related}
\paragraph{Static and dynamic regret.}
The moving-comparator projection analysis originates with \citet{zinkevich2003}, and AdaGrad adapts stepsizes to observed gradient energy \citep{duchi2011}. Ader achieves the optimal exact-gradient path-length dependence through adaptive aggregation and common linear surrogates \citep{zhang2018}. Under heavy-tailed feedback, \citet{liu2026} obtains the optimal static expected rate with classical algorithms, including parameter-free AdaGrad. \citet{aggarwal2026} establishes the joint full-horizon dynamic rate, its parameter adaptation, and the matching polynomial benchmark. We use its restarted-expert inequality at the point where it applies, and give the additional window-alignment calculation in Appendix~\ref{app:prelim}. The present local master uses conditional noise moments; the full-horizon energy argument needs only marginal moments.

\paragraph{Interval and strongly adaptive regret.}
Adaptive regret and strongly adaptive learning provide comparisons on all subintervals \citep{hazan2009,daniely2015}. Geometric covers and sleeping aggregation reduce these objectives to local expert comparisons \citep{jun2017}. Dynamic and interval objectives have already been combined by \citet{cutkosky2020} and \citet{zhang2020adaptive}. In particular, the latter's AOA algorithm gives the exact-gradient specialization $O(\mathcal B_I(\Lambda_I+L))$ in Table~\ref{tab:comparison}. Appendix~\ref{app:exact-comparison} gives the linear-surrogate normalization used to express that result with the same $G,D,P_I$ notation. Our results extend local comparisons to unknown conditional finite-$p$ noise, separating the mean-gradient and noise powers and identifying a statistical interval cost when the parameters are supplied.

\citet{xie2026} give gradient-variation and small-loss interval guarantees, including dynamic comparators and adaptation to unknown Lipschitz and smoothness constants. Its problem-dependent analysis assumes smooth nonnegative losses, uses exact first-order information and function values in its scheduling, and retains maximum observed input scales in its adaptive terms. These results refine local environmental variation. Our guarantee instead quantifies comparator movement and central noise moments under gradient-only feedback, including infinite-variance noise.

\paragraph{Heavy tails and adaptive aggregation.}
Finite-moment stochastic convex optimization is studied by \citet{vural2022}, while \citet{zhangcutkosky2022} treat comparator-norm adaptation in high probability. \citet{moulin2025} give optimal adaptive full-information expert aggregation under finite second moments and show that maximum-loss terms can dominate a nominally lower-order contribution. Its full-horizon expert guarantee and our finite-$p$ sleeping comparisons have different feedback and localization requirements. This distinction is relevant even at $p=2$.

Multi-rate excess-loss bounds, relative-entropy comparisons, and sleeping-expert reductions are established tools \citep{gaillard2014,luo2015}. Adaptive mixture and scale analyses appear in \citet{derooij2014,orabona2018}. Polynomial priors are used in Ader \citep{zhang2018}; we apply them to distances from the longest window and restart scales. The complete normalization and selected-cost calculation appear in Proposition~\ref{prop:prior}. The calibration test uses nonnegative-supermartingale reasoning \citep{howard2020}; its expected overshoot and its combination with a constant number of relevant scale phases are proved here. The lower bound uses a ternary oracle and a null-versus-window change of measure, with the full calculation in Section~\ref{sec:price}.

\begin{table}[t]
\centering\small
\caption{Key results and relation to closely related work. All rows use $\mathcal B_I(q)$ from \eqref{eq:benchmark}; its argument is the total comparator and adaptation complexity. Write $L=\log(2T)$ and $\chi_T(n)=1+\log(T/n)$. Upper bounds may be capped by $GD|I|$.}
\label{tab:comparison}
\renewcommand{\arraystretch}{1.12}
\begin{tabular}{@{}>{\raggedright\arraybackslash}p{.245\textwidth}>{\raggedright\arraybackslash}p{.305\textwidth}>{\raggedright\arraybackslash}p{.12\textwidth}>{\raggedright\arraybackslash}p{.235\textwidth}@{}}
\toprule
Result & Evaluation and feedback & Statistics supplied & Expected regret\\
\midrule
Ader \citep{zhang2018} & Full horizon, moving; exact gradients & $G$ & $O(\mathcal B_{[T]}(\Lambda_{[T]}))\big|_{\sigma=0}$\\[3pt]
AdaGrad \citep{liu2026} & Full horizon, static; finite-$p$ noise & None & $O(\mathcal B_{[T]}(1))$\\[3pt]
\citet{aggarwal2026} & Full horizon, moving; marginal moments & None & $O_p(\mathcal B_{[T]}(\Lambda_{[T]}))$\\[3pt]
AOA \citep{zhang2020adaptive} & Every interval, moving; exact gradients$^{\dagger}$ & $G$ & $O(\mathcal B_I(\Lambda_I+L))\big|_{\sigma=0}$\\[3pt]
Corollary~\ref{cor:static} & Every fixed interval, static; conditional moments & None & $O(\mathcal B_I(1+L^2))$\\[3pt]
Theorem~\ref{thm:main} & Every fixed interval, moving; conditional moments & None & $O(\mathcal B_I(\Lambda_I+L^2))$\\[3pt]
Theorem~\ref{thm:known} & Every fixed interval, moving; conditional moments & $G,\sigma,p$ & $O(\mathcal B_I(\Lambda_I+\chi_T(n)))$\\
\bottomrule
\end{tabular}
\par\smallskip\raggedright\footnotesize
The domain, its diameter, and the horizon are available throughout. ``None'' means none of $G,\sigma,p$; path length is also unsupplied. The $O_p$ row permits a coefficient depending only on $p$; its noise coefficient is $O(1+\log(p/(p-1)))$. Constants in our upper bounds and the static AdaGrad row are universal over $p$. $^{\dagger}$The AOA row is its normalized linear-surrogate specialization, proved in Appendix~\ref{app:exact-comparison}. The supplied-statistics row with $\Lambda_I=1$ gives the static specialization. Corollary~\ref{cor:statistical-price} establishes its logarithmic noise lower bound under the explicit full-horizon requirement and parameter regime.
\end{table}

\section{Setup and Information Model}\label{sec:setup}
Let $\X\subset\R^d$ be nonempty, closed, and convex, with Euclidean diameter $D\in(0,\infty)$, and let $T\ge2$. Euclidean projection onto $\X$ is available. On round $t$, the learner chooses a predictable $x_t\in\X$ and receives only one stochastic subgradient $g_t$ at that point. Let $\F_0$ contain its random seed and $\F_t=\sigma(\F_0,g_1,\ldots,g_t)$. The convex loss $\ell_t$ and a measurable selected subgradient map may be fixed in advance or chosen measurably from $\F_{t-1}$. Each loss admits a measurable selected subgradient map whose norm is at most $G$ on $\X$. Write $v_t\in\partial\ell_t(x_t)$ for its value at the played point and $\epsilon_t=g_t-v_t$. For deterministic but unknown $G,\sigma\ge0$ and $p\in(1,2]$, assume almost surely
\begin{equation}\label{eq:model}
 \E[g_t\mid\F_{t-1}]=v_t,\qquad
 \|v_t\|\le G,\qquad
 \E[\|\epsilon_t\|^p\mid\F_{t-1}]\le\sigma^p.
\end{equation}
The selected subgradient bound makes each loss $G$-Lipschitz on $\X$. The conditional law may be asymmetric and depend on the past. Feedback consists of the gradient vector alone.

For a deterministic interval $I=[s,e]$ and deterministic comparator sequence $u_t\in\X$, define
\begin{align}\label{eq:regret}
 \Reg_I(u)&=\sum_{t=s}^e[\ell_t(x_t)-\ell_t(u_t)],\notag\\
 P_I&=\sum_{t=s+1}^e\|u_t-u_{t-1}\|,
 &\Lambda_I&=1+P_I/D.
\end{align}
Thus $1\le\Lambda_I\le |I|$. Conditional unbiasedness and convexity give
\begin{align}\label{eq:linearization}
 \E\Reg_I(u)&\le\E\sum_{t\in I}
 \langle g_t,x_t-u_t\rangle,\notag\\
 \Reg_I(u)&\le GD|I|\quad\text{almost surely}.
\end{align}
Regret differences are bounded in absolute value by $GD$, and the linearized terms are integrable because $\E\|g_t\|\le G+\sigma$. The comparisons hold separately for every fixed $I,u$ on the same decisions. The parameters $G,\sigma$ are uniform global bounds. The local guarantee depends on the evaluation interval through its length and comparator movement, rather than an additive realized energy accumulated before its start.

\paragraph{General sleeping comparisons.}
The master theorem is stated for a finite set of records $i\in\mathcal M$ with fixed prior masses $\pi_i>0$, $\sum_i\pi_i=1$. Record $i$ has a predictable availability $\alpha_{t,i}\in\{0,1\}$ and a predictable feasible prediction $y_{t,i}$ when awake. At least one record is awake on every round. For the comparison to $i$, a deterministic integer $n_i\in[1,T]$ satisfies
\begin{equation}\label{eq:awake-budget}
 \sum_{t=1}^T\alpha_{t,i}\le n_i\quad\text{almost surely}.
\end{equation}
The budget is used only in the proof. Availability can depend on observed feedback. Theorem~\ref{thm:master} also permits a fixed distribution over records with an almost-sure budget on its weighted activity. This pathwise budget controls its possible dependence on calibration.

\paragraph{Dyadic windows and base trajectories.}
Set $N=\lceil\log_2T\rceil$. For $0\le k\le N$, let $\DD_k$ consist of the labeled supports
$[a2^k+1,(a+1)2^k]\cap[T]$, $0\le a<\lceil T/2^k\rceil$, and put $\DD=\bigsqcup_k\DD_k$. Labels distinguish coinciding truncated supports. Every fixed interval is a disjoint union of at most two untruncated dyadic windows of each size, and for $1/2\le\gamma\le1$ its maximal dyadic cover obeys
\begin{equation}\label{eq:dyadic-power}
 \sum_{J\text{ in the cover}}|J|^\gamma\le7|I|^\gamma.
\end{equation}
Maintain $N+1$ shared-gradient AdaGrad trajectories. Trajectory $j$ restarts at the left endpoints of $\DD_j$, from a fixed $x_0\in\X$, and within a block starting at $s$ uses
\begin{equation}\label{eq:base-update}
 V_t^{(j)}=\sum_{r=s}^t\|g_r\|^2,\qquad
 x_{t+1}^{(j)}=\Pi_\X\!\left(x_t^{(j)}-\frac{Dg_t}{\sqrt{2V_t^{(j)}}}\right).
\end{equation}
A zero energy leaves the point unchanged. Every trajectory receives the same linear surrogate $x\mapsto\langle g_t,x\rangle$; the oracle is queried only at $x_t$.

For the interval algorithm, use records $i=(j,J)$ with $J\in\DD_k$ and $0\le j\le k$. Set $m_k=\lceil T/2^k\rceil$ and
\begin{align}
 \alpha_{t,(j,J)}&=\ind\{t\in J\},\qquad y_{t,(j,J)}=x_t^{(j)},\notag\\
 \tau_k&=\frac{N+2}{N+1}
 \frac1{(N-k+1)(N-k+2)},\notag\\
 \zeta_{j\mid k}&=\frac{k+2}{k+1}
 \frac1{(k-j+1)(k-j+2)},\notag\\
 \pi_{(j,J)}&=\frac{\tau_k\zeta_{j\mid k}}{m_k}.\label{eq:pool}
\end{align}
The prior favors long windows and, within a window scale, long restart lengths. Locations are uniform within each scale. Its three factors keep these selection costs separate. Both the parameter-free and supplied-statistics algorithms use this same prior.

\begin{proposition}[Normalization and selected prior cost]\label{prop:prior}
The masses in \eqref{eq:pool} sum to one, and $\pi_{(N,[T])}\ge1/4$. For an untruncated $J\in\DD_k$ of length $n_J=2^k$ and $j^*=\lfloor\log_2(n_J/\Lambda_J)\rfloor$,
\begin{equation}\label{eq:prior-cost}
 1+\log(1/\pi_{(j^*,J)})\le c_T(n_J,\Lambda_J),
\end{equation}
where, for $1\le n\le T$ and $\lambda\ge1$,
\begin{align}\label{eq:prior-refinement}
 c_T(n,\lambda)={}&1+\log2+\log(T/n)\notag\\
 &+2\log\bigl(3+\log_2(T/n)\bigr)\notag\\
 &+2\log\bigl(3+\log_2\lambda\bigr).
\end{align}
All records satisfy $\log(1/\pi_i)=O(L)$.
\end{proposition}
\begin{proof}
For every integer $m\ge0$, $\sum_{r=0}^m[(r+1)(r+2)]^{-1}=(m+1)/(m+2)$. This normalizes $\tau$ and each $\zeta_{\cdot\mid k}$. There are $m_k$ windows at scale $k$, proving normalization of $\pi$. The top-scale and top-restart masses are each at least $1/2$, while $m_N=1$, proving the root bound.

For an untruncated window, put $r=N-k<1+\log_2(T/n_J)$ and $d=k-j^*=\lceil\log_2\Lambda_J\rceil\le1+\log_2\Lambda_J$. Since $m_k\le2T/n_J$ and the normalizing factors in $\tau,\zeta$ exceed one,
\begin{multline*}
 1+\log(1/\pi_{(j^*,J)})
 \le1+\log(2T/n_J)\\
 {}+2\log(r+2)+2\log(d+2),
\end{multline*}
which proves \eqref{eq:prior-cost}. For every labeled record, including truncated supports, $m_k\le T$, $r,d\le N$, so $\log(1/\pi_i)\le\log T+4\log(N+2)=O(L)$.
\end{proof}
The leading location cost is $\log(T/n)$; scale and restart selection add iterated logarithms. To see why location has a different role, consider $K$ disjoint candidate windows and sum all record masses assigned to each. At least one window receives mass at most $1/K$, giving a prior cost at least $\log K$ for that window. This elementary allocation fact explains the location term. The statistical lower bound in Section~\ref{sec:price} separately establishes its effect on regret.

\begin{proposition}[Aligned base comparison]\label{prop:base}
For a fixed $J\in\DD$, choose $j^*=\lfloor\log_2(|J|/\Lambda_J)\rfloor$. Then
\begin{equation}\label{eq:base-bound}
 \E\sum_{t\in J}\langle g_t,x_t^{(j^*)}-u_t\rangle
 \le3\left[GD\sqrt{|J|\Lambda_J}+\sigma D|J|^{1/p}\Lambda_J^\rho\right].
\end{equation}
\end{proposition}
This applies the restarted-AdaGrad inequality of \citet[Lemma 4.3]{aggarwal2026} to an aligned window. Appendix~\ref{app:prelim} gives the complete alignment and moment calculation, including the constant three, and proves \eqref{eq:dyadic-power}.

\section{Algorithm and Interval-Dynamic Guarantees}\label{sec:algorithm}
\subsection{The sleeping master}
The master below applies to the general records in Section~\ref{sec:setup}; \eqref{eq:pool} gives the interval instantiation. Set
\begin{align}\label{eq:constants}
 Q&=\lceil8\log_2(2T)\rceil,
 &\delta&=T^{-6}, &\kappa&=8,\notag\\
 q_{i,l}&=\pi_i/(Q+1),
 &&0\le l\le Q.
\end{align}
A positive observation scale $b$ defines rates $\eta_l=2^{-l}/(8Db)$. Pair-rate wealths $W(i,l)$ start at one, and the conceptual potential is $\Phi=\sum_{i,l}q_{i,l}W(i,l)$. Rate $l$ has a penalty guess $a_l=\delta$ and counter $s_l=0$. Sleeping wealths remain unchanged; in the interval implementation unborn wealths are one and expired contributions are frozen.

Write $\Phi^-$ for the pre-round potential. For awake records set
\begin{align}\label{eq:prediction}
 u_{t,i,l}&=\frac{q_{i,l}W(i,l)}{\Phi^-(1+2a_l)},\notag\\
 A_{t,l}&=\sum_{i\text{ awake}}u_{t,i,l},\notag\\
 x_t&=\frac{\sum_{i,l\text{ awake}}u_{t,i,l}\eta_l y_{t,i}}
 {\sum_{i,l\text{ awake}}u_{t,i,l}\eta_l}.
\end{align}
After observing $g_t$, put $r_{t,i}=\langle g_t,x_t-y_{t,i}\rangle$, $z_{t,i,l}=\eta_l r_{t,i}$, and update awake wealths by
\begin{align}\label{eq:wealth}
 W^+(i,l)&=W(i,l)\frac{F(z_{t,i,l})}{1+2a_l},\notag\\
 F(z)&=1+\clip(z,-1/2,1/2).
\end{align}
Let $\curv(z)=(F(z)-1-z)_+$ and form, using pre-round guesses and exposures,
\begin{align}
 d_{t,l}&=\sum_{i\text{ awake}}u_{t,i,l}\curv(z_{t,i,l}),
 &d_t&=\sum_l d_{t,l},\notag\\
 \Gamma_t&=\sum_l2a_lA_{t,l},
 &C_t&=\Phi^+/\Phi^- -1+\Gamma_t,\label{eq:allocation}\\
 Z_{t,l}&=(C_t)_+d_{t,l}/d_t\quad(d_t>0),
 &Z_{t,l}&=0\quad(d_t=0).\notag
\end{align}
Update $s_l\leftarrow s_l+Z_{t,l}-2a_lA_{t,l}$. If $s_l>\kappa$, double $a_l$ and reset that counter to zero. Wealths retain their values during a counter reset.

The scale changes only after completing the round's wealth and counter updates:
\begin{align}\label{eq:trigger}
 \|g_t\|>T^3b\quad\Longrightarrow\quad
 b&\leftarrow2\|g_t\|/T^2,\notag\\
 W&\leftarrow1,\quad\Phi\leftarrow1,\quad
 a_l\leftarrow\delta,\quad s_l\leftarrow0.
\end{align}
Initially $b=0$. Until the first nonzero observation, play the prior mixture of awake predictions and skip wealth and counter updates; that observation initializes $b$ through \eqref{eq:trigger}. A phase has constant positive scale and includes its final trigger round.

\begin{algorithm}[t]
\caption{Parameter-free sleeping master with expected calibration}\label{alg:master}
\KwInput{$T\ge2$, $\X$, $D$, positive prior $\pi$, predictable awake predictions}
Set \eqref{eq:constants}, $b=0$, $W=1$, $\Phi=1$, $a_l=\delta$, $s_l=0$\;
\For{$t=1,\ldots,T$}{
 Obtain the awake predictions; in \eqref{eq:pool}, restart the scheduled AdaGrad trajectories\;
 \eIf{$b=0$}{
 Play the prior mixture of awake predictions; observe $g_t$\;
 }{
 Set $\eta_l=2^{-l}/(8Db)$ and compute \eqref{eq:prediction} using pre-round quantities\;
 Play $x_t$, observe $g_t$, and form all awake excesses $r_{t,i}$\;
 Compute the updated wealths and potential from \eqref{eq:wealth}\;
 Form $Z_{t,l}$ and every $Y_{t,l}=Z_{t,l}-2a_lA_{t,l}$ using \eqref{eq:allocation} and old guesses\;
 Store the updated wealths and potential\;
 \For{$l=0,\ldots,Q$}{
 Set $s_l\leftarrow s_l+Y_{t,l}$\;
 \If{$s_l>8$}{Set $a_l\leftarrow2a_l$ and $s_l\leftarrow0$\;}
 }
 }
 In the interval instantiation, update every base trajectory with the raw $g_t$ using \eqref{eq:base-update}\;
 \If{$\|g_t\|>T^3b$}{Apply the scale change and master reset in \eqref{eq:trigger}\;}
}
\end{algorithm}

Only the scalar wealth score is truncated. Every AdaGrad trajectory receives the raw gradient. The first-order comparison remains untruncated, and the logarithmic defect below accounts for the score modification. Scale changes leave the base trajectories and their deterministic restarts untouched. Positive rescaling of all observed gradients leaves the decisions unchanged: $b$ rescales, the products $\eta_l r_{t,i}$ do not, and the normalized base updates are invariant.

In \eqref{eq:pool}, there are at most $(N+1)(N+2)/2$ awake pairs. The implementation uses one oracle query, $O(\log T)$ projections, and $O(d\log T+\log^3T)$ arithmetic and memory per round, excluding projection cost. Unborn mass, frozen expired contributions, and phase stamps implement the conceptual pool without scanning it at a reset. General reawakening specialists require storage of their accumulated wealths, rather than the interval-specific active-window bound. These counts concern exact-real arithmetic.

\subsection{Calibration inequalities and the local theorem}\label{sec:calibration}
Set $H=DG$, $S=D\sigma$, and $\defect(z)=(z-\log F(z))_+$. The following two elementary score bounds are the basis of the moment analysis:
\begin{align}\label{eq:defect-bound}
 0&\le q(z)\le\min\{2z^2,|z|\},\notag\\
 q(a+z)&\le4a^2+4|z|^p,
 \qquad q\in\{\curv,\defect\}.
\end{align}
Consequently, every predictable rate $\eta>0$ and awake prediction satisfy
\begin{equation}\label{eq:h}
 \E[q(\eta r_{t,i})\mid\F_{t-1}]
 \le h(\eta):=4\eta^2H^2+4\eta^pS^p.
\end{equation}
Appendix~\ref{app:score} proves \eqref{eq:defect-bound} and the allocation identities used next. Under asymmetric noise, the positive tangent excess can create potential drift. Calibration accounts for this drift through \eqref{eq:h}.

\begin{lemma}[Bounded allocation and a pathwise potential budget]\label{lem:potential}
On each positive-scale round,
\begin{align}\label{eq:response}
 0&\le Z_{t,l}\le2,\notag\\
 \E[Z_{t,l}\mid\F_{t-1}]&\le A_{t,l}h(\eta_l),\notag\\
 \log(\Phi^+/\Phi^-)&\le
 \sum_l(Z_{t,l}-2a_lA_{t,l}).
\end{align}
Put $B_T=\lceil9\log_2(2T)\rceil$. On every phase, without a probabilistic event, every guess doubles at most $B_T$ times and
\begin{equation}\label{eq:potential-budget}
 \log\Phi\le(Q+1)(B_T+1)(\kappa+2)=O(L^2).
\end{equation}
The bound includes the potential immediately after an unbounded trigger observation and before its reset.
\end{lemma}
\begin{proof}
The mixture cancels $\sum_{i,l}u_{t,i,l}z_{t,i,l}$ pathwise. Thus
$\Phi^+/\Phi^-=1+C_t-\Gamma_t$ and $(C_t)_+\le d_t$.
Because $F\le3/2$, the potential ratio is at most $3/2$, while $0\le\Gamma_t\le1$. Hence $\sum_lZ_{t,l}=(C_t)_+<2$ and $Z_{t,l}\le d_{t,l}$. Taking conditional expectations gives the response bound, and $\log x\le x-1$ gives the log increment.

On a nontrigger round, $|z_{t,i,l}|\le T^3/8$. Therefore
$Z_{t,l}\le d_{t,l}\le (T^3/8)A_{t,l}$, so the counter increment is nonpositive whenever $a_l\ge T^3/16$. A segment beginning with such a guess has counter at most zero until its possible trigger round. That final round adds at most two, which cannot exceed $\kappa=8$. It cannot produce another doubling. Since doubling below $T^3/16$ yields at most $T^3/8$, and $\delta<T^3/16$, every guess is at most $T^3/8$. Starting from $T^{-6}$ permits at most $9\log_2 T$ doublings.

A completed counter segment ends in $(\kappa,\kappa+2]$; an unfinished segment has counter at most $\kappa$, with negative balances retained. Summing all increments from the phase's initial potential one gives \eqref{eq:potential-budget}. No moment bound is used in this last argument.
\end{proof}

\begin{lemma}[Expected maximum calibration excess]\label{lem:overshoot}
Condition on the start of a positive-scale phase. Choose a rate index measurably at that start, and write $h=h(\eta_l)$. Let $a_{\max}$ be its largest guess during the phase, including the final update. Then
\begin{equation}\label{eq:overshoot}
 \E[(a_{\max}-\delta)_+\mid\text{phase-start information}]\le4h.
\end{equation}
This remains valid for a phase ended by an observation-dependent trigger and for $h<\delta$. If $h=0$, the guess remains $\delta$ almost surely.
\end{lemma}
\begin{proof}
Consider a counter segment with fixed guess $a\ge h>0$, starting at zero. Put $r=a/h\ge1$ and $\lambda=(1+\log r)/2$. For $0\le Z\le2$, convexity gives
$e^{\lambda Z}\le1+(e^{2\lambda}-1)Z/2$. Predictability of $A$, together with \eqref{eq:response}, yields
\begin{multline*}
 \E[e^{\lambda(Z-2aA)}\mid\F_{t-1}]\\
 \le\exp\!\left\{Aa\left[\frac e2-1-\log r
 -\frac1{2r}\right]\right\}\le1.
\end{multline*}
The last bracket is negative at $r=1$ and has negative derivative for $r\ge1$. The exponential counter, stopped at the segment end, is a nonnegative supermartingale. Its conditional crossing probability is therefore at most
\begin{equation}\label{eq:crossing}
 e^{-\lambda\kappa}=q_0(h/a)^4,\qquad q_0=e^{-4}<1/4.
\end{equation}
This is the standard exponential maximal argument, with the parameter chosen to retain the dependence on $h/a$.

If $h\ge\delta$, the first dyadic guess $a_*\ge h$ lies in $[h,2h]$. Successive crossings at valid guesses have conditional probability at most $q_0$. Thus
$\Pr(a_{\max}\ge2^k a_*\mid\text{start})\le q_0^k$ and
\[
 \E a_{\max}\le a_*\left(1+\frac{q_0}{1-2q_0}\right)<4h.
\]
If $0<h<\delta$, retain the factor $(h/\delta)^4$ for the first crossing and use $q_0$ for subsequent ones. Summing the tail of the dyadic levels gives
\[
 \E(a_{\max}-\delta)\le
 \frac{\delta q_0}{1-2q_0}(h/\delta)^4\le h.
\]
All expectations here are conditional on the phase start. Segment starts and ends are stopping times, so the same conditional arguments apply successively. When $h=0$, nonnegativity and zero conditional expectation force $Z=0$, and no crossing occurs.
\end{proof}

\begin{theorem}[Prior-sensitive predictable sleeping aggregation]\label{thm:master}
Under \eqref{eq:model}, fix a deterministic distribution $\nu$ on $\mathcal M$. Define its predictable activity $w_t=\sum_i\nu_i\alpha_{t,i}$ and suppose that, for deterministic $1\le n_\nu\le T$,
\begin{equation}\label{eq:mixture-budget}
 \sum_{t=1}^T w_t\le n_\nu\quad\text{almost surely}.
\end{equation}
Algorithm~\ref{alg:master} satisfies
\begin{align}
 &\E\sum_t\sum_i\nu_i\alpha_{t,i}
 \langle v_t,x_t-y_{t,i}\rangle\notag\\
 &\quad\le\min\left\{Hn_\nu,\ C\left[
 H\sqrt{n_\nu\mathcal K_\nu}
 +S n_\nu^{1/p}\mathcal K_\nu^\rho\right]\right\},
 \label{eq:master-bound}\\
 \mathcal K_\nu&=1+\KL(\nu\|\pi)+L^2,\qquad
 \KL(\nu\|\pi)=\sum_{i:\nu_i>0}
 \nu_i\log\frac{\nu_i}{\pi_i}.\notag
\end{align}
The constant is universal. The learner uses none of $G,\sigma,p,\nu,n_\nu$. In particular, $\nu=e_i$ yields the individual-record bound with $n_\nu=n_i$ and $\mathcal K_i=1+\log(1/\pi_i)+L^2$.
\end{theorem}
The relative-entropy dependence retains the benefit of substantial prior mass on several useful experts. For example, let $A$ be a fixed set of records sharing support $J$, and choose $\nu_i=\pi_i/\pi(A)$ on $A$. Then $n_\nu=|J|$ and $\KL(\nu\|\pi)=\log(1/\pi(A))$. The theorem compares to their prior-weighted average prediction on $J$. This is the standard distribution-comparison interpretation \citep{luo2015}, now under the conditional finite-$p$ model and expected calibration.

\paragraph{Proof structure.}
Appendix~\ref{app:master} gives the full proof. For positive wealths, the relative-entropy variational inequality gives
\begin{equation}\label{eq:variational-outline}
 \sum_i\nu_i\log W(i,l)
 \le\KL(\nu\|\pi)+\log(Q+1)+\log\Phi.
\end{equation}
Keeping $\nu$ inside this step preserves its entropy saving. Averaging individual-record bounds would instead charge $\sum_i\nu_i\log(1/\pi_i)$.

Let $M_t=\max_{r\le t}\|g_r\|$. The trigger implies $b_t\ge M_{t-1}/T^3$. Predictable rounds with $b_t<G/T^4$ contribute at most $2H+Sn_\nu^{1/p}$ in expectation. Except on an event of probability at most $2T^{-3}$, at most nine larger-scale phases occur. A grid rate is selected predictably in each of these phases, with $\eta\le\eta_*$ and $1/\eta\le2/\eta_*+8Db$, for a deterministic comparison target $\eta_*$. The potential budget and \eqref{eq:variational-outline} give $K_\nu=O(\KL(\nu\|\pi)+L^2)$ per phase. Lemma~\ref{lem:overshoot} bounds the expected sum of normalized penalty excesses over those phases by $36$.

The weighted activity in \eqref{eq:mixture-budget} controls both the defects and a single martingale sum. The resulting tradeoff, apart from the explicitly absorbed remainder, is
\begin{equation}\label{eq:tradeoff}
 \frac{2\mathcal A_\nu}{\eta_*}+400n_\nu H^2\eta_*+
 400n_\nu S^p\eta_*^{p-1},\qquad \mathcal A_\nu=9K_\nu.
\end{equation}
Balancing its two moment terms proves \eqref{eq:master-bound}. The target uses the deterministic weighted activity budget, so rate selection remains predictable even when availability depends on the observed history. The $L^2$ term comes from the pathwise calibration budget in Lemma~\ref{lem:potential}; the prior cost remains separate in \eqref{eq:master-bound}.

\subsection{Dynamic regret and temporal concentration}
\begin{theorem}[Additive interval adaptation]\label{thm:main}
Under \eqref{eq:model}, Algorithm~\ref{alg:master} with \eqref{eq:pool} and \eqref{eq:base-update} satisfies, for every fixed interval $I$ and fixed comparator path,
\begin{align}
 \E\Reg_I(u)&\le\min\{GDn,\ C[\mathcal B_I(\Lambda_I)+GD\sqrt n\,L+\sigma Dn^{1/p}L^{2\rho}]\},\label{eq:main}\\
 \E\Reg_I(u)&\le\min\!\left\{GDn,\ C'\!\left[GD\sqrt{n(\Lambda_I+L^2)}
 +\sigma Dn^{1/p}(\Lambda_I+L^2)^\rho\right]\right\}.\label{eq:main-effective}
\end{align}
The constants are universal for $1<p\le2$. The same decisions support every stated comparison, using none of $G,\sigma,p,I,P_I$.
\end{theorem}
\begin{proof}
Take the deterministic maximal dyadic cover of $I$. For each piece $J$, select the fixed restart index from Proposition~\ref{prop:base}. In \eqref{eq:pool}, $\mathcal K_i=O(L^2)$, uniformly over records. The shared-gradient identity, conditional unbiasedness, Theorem~\ref{thm:master}, and Proposition~\ref{prop:base} add the master and base comparisons on $J$.

Write $x=P_I/D$. The power sums \eqref{eq:dyadic-power}, $\sum_JP_J\le P_I$, Cauchy--Schwarz, and H\"older yield
\begin{align*}
 \sum_J\sqrt{|J|\Lambda_J}&\le7\sqrt n+\sqrt{nx}\le8\sqrt{n\Lambda_I},\\
 \sum_J|J|^{1/p}\Lambda_J^\rho&\le7n^{1/p}+n^{1/p}x^\rho\le8n^{1/p}\Lambda_I^\rho.
\end{align*}
The static master terms sum by \eqref{eq:dyadic-power}. This proves \eqref{eq:main}; \eqref{eq:linearization} supplies its ceiling. For $0<a\le1$, $x^a+y^a\le2(x+y)^a$, giving \eqref{eq:main-effective} with a universal constant.
\end{proof}

\begin{corollary}[Static comparisons on every fixed interval]\label{cor:static}
For every fixed interval $I$ and point $u\in\X$, the same learner satisfies
\begin{equation}\label{eq:static}
 \E\Reg_I(u)\le\min\{GDn,\ C\mathcal B_I(1+L^2)\}
 \le C'\left[GD\sqrt n\,L+\sigma Dn^{1/p}L^{2\rho}\right].
\end{equation}
\end{corollary}
\begin{proof}
A constant comparator has $P_I=0$ and $\Lambda_I=1$. Apply Theorem~\ref{thm:main} and use $L\ge\log4>1$.
\end{proof}
This specialization makes the cost of local evaluation visible without comparator movement. The full-horizon heavy-tailed static scale is $\mathcal B_{[T]}(1)$; calibration adds $L^2$ to the interval comparison. With supplied statistics, Theorem~\ref{thm:known} gives $\mathcal B_I(1+\chi_T(n))$, recovering the full-horizon static scale when $n=T$.

\begin{corollary}[Deterministic partitions]\label{cor:partition}
Let $U(J)$ be the right side of \eqref{eq:main}, including its own linear ceiling. For a fixed $I,u$, the same algorithm satisfies
\begin{equation}\label{eq:partition}
 \E\Reg_I(u)\le\min_{\mathcal Q\text{ a deterministic partition of }I}
 \sum_{J\in\mathcal Q}U(J).
\end{equation}
\end{corollary}
\begin{proof}
For each deterministic partition, add the expected bounds for its pieces. Their numerical minimum also bounds the same expectation. The partition selection is outside expectation.
\end{proof}
For example, in the noiseless case with movement confined to an initial episode of length $m$, followed by a constant comparator, \eqref{eq:partition} gives $O(GD[m+\sqrt nL])$. With $T=n$, $m\asymp n^{2/3}$, and $P_I\asymp Dm$, this is $O(GDn^{2/3})$, whereas the total-path expression is of order $GDn^{5/6}$. The comparison concerns the two upper bounds. Local evaluation retains information about the temporal concentration of movement that is absent from its total path length alone.

\section{The Statistical Cost of Interval Adaptation}\label{sec:price}
The polynomial fixed-window benchmark follows from \citet[Theorem 5.1]{aggarwal2026}: for every fixed length $n\ge2$ and path budget $P$, every learner has an admissible fixed instance with expected regret at least
\begin{equation}\label{eq:inherited-lower}
 \frac1{20}\min\!\left\{GD\sqrt{n(1+P/D)}+\sigma Dn^{1/p}(1+P/D)^\rho,\ GDn\right\}.
\end{equation}
To place it in a longer run, use zero losses and observations outside the window. The exact-gradient and rare-revelation constructions in that theorem satisfy conditional moments; the entering learner state is simply an independent initial seed. This supplies the polynomial benchmark on every bounded convex domain. The results below address the additional cost of locating an unannounced evaluation window.

\subsection{A full-horizon-preserving upper bound with supplied statistics}
Define $\chi_T(n)=1+\log(T/n)$. The supplied-statistics variant uses the same bounded score and shared-gradient bases, but replaces calibration by known penalties. Both learners use the nonuniform prior in \eqref{eq:pool}. Its complete algorithm and proof are in Appendix~\ref{app:known}.
\begin{theorem}[Supplied-statistics interval guarantee]\label{thm:known}
When $G,\sigma,p$ are supplied, with $G>0$, there is a single predictable learner satisfying
\begin{equation}\label{eq:known}
 \E\Reg_I(u)\le\min\!\left\{GDn,\ C\!\left[
 GD\sqrt{n(\Lambda_I+\chi_T(n))}
 +\sigma Dn^{1/p}(\Lambda_I+\chi_T(n))^\rho\right]\right\}
\end{equation}
for every fixed $I,u$. The separated form has master terms
$C[GD\sqrt{n\chi_T(n)}+\sigma Dn^{1/p}\chi_T(n)^\rho]$ in addition to $C\mathcal B_I(\Lambda_I)$. It uses one gradient query, $O(\log T)$ projections, and $O(d\log T+\log^2T)$ other arithmetic and storage per round. Neither the evaluation interval nor comparator movement is supplied. All constants are universal over $p$. More explicitly, with $c=c_T(n,\Lambda_I)$ from \eqref{eq:prior-refinement}, the same learner satisfies
\begin{equation}\label{eq:known-refined}
 \E\Reg_I(u)\le\min\!\left\{GDn,\ C\left[\mathcal B_I(\Lambda_I)
 +GD\sqrt{nc}+\sigma Dn^{1/p}c^\rho\right]\right\}.
\end{equation}
\end{theorem}
The analysis makes the complete sleeping potential a nonnegative supermartingale, so its expected logarithm at a fixed endpoint is nonpositive. Proposition~\ref{prop:prior} retains one leading location logarithm, with iterated-logarithmic scale and restart costs. Formula~\eqref{eq:known-refined} preserves these costs before the simpler bound \eqref{eq:known} absorbs them into $\Lambda_I+\chi_T(n)$. The weighted cover preserves $\chi_T(n)$ rather than replacing it by $\log T$. At $I=[T]$, \eqref{eq:known} is the full-horizon polynomial benchmark with universal constants. The fully parameter-free algorithm in Theorem~\ref{thm:main} instead pays its displayed $L^2$ cost on that interval.

\begin{corollary}[Full-horizon static rate with supplied statistics]\label{cor:known-static}
For every fixed point $u\in\X$, the learner of Theorem~\ref{thm:known} satisfies
\begin{equation}\label{eq:known-static}
 \E\Reg_{[T]}(u)\le\min\{GDT,\ 40\mathcal B_{[T]}(1)\}.
\end{equation}
\end{corollary}
The root record has prior mass at least $1/4$ and runs AdaGrad without a restart. Direct comparison to this record gives the constant in \eqref{eq:known-static}, uniformly over $p$; the full calculation appears in Appendix~\ref{app:known}. This explains the role of favoring long windows even for a static comparator.

\subsection{A regret--testing tradeoff}
The next theorem quantifies the conflict between small regret on a persistent environment and rapid identification of an unannounced exceptional window. It applies even when all statistical parameters and the collection of candidate windows are known.

\begin{theorem}[From full-horizon regret to a hard interval]\label{thm:testing}
Let $T=Kn$, where $K\ge2$ and $n\ge1$ are integers. Fix $\sigma>0$, $p\in(1,2]$, and $q\in(0,1)$, and set
\begin{equation}\label{eq:testing-parameters}
 \Delta=\frac{\sigma}{8}q^\rho\le G,\qquad c_q=\frac q2\log3.
\end{equation}
For every learner there is a fixed constant linear-loss instance satisfying \eqref{eq:model}, with gradient norm $\Delta$, called the null instance. Suppose its expected regret to its constant minimizer is at most $B$, where $0<B<D\Delta T$. Then one of the $K$ consecutive length-$n$ windows is the exceptional window in another admissible fixed instance, and on that window, against its constant minimizer,
\begin{equation}\label{eq:testing-lower}
 \E\Reg_I(u)\ge D\Delta n
 \left[1-\frac{nc_q+\log2}{\log(D\Delta T/B)}\right]_+.
\end{equation}
Both instances have independent oracle outputs and satisfy the same conditional central-moment bound. When the null regret is zero, the exceptional-window regret is $D\Delta n$.
\end{theorem}
\begin{proof}
\textbf{Geometry and oracle.} Use a diameter direction as in \citet[Lemma 5.2]{aggarwal2026}. There exist $a,b\in\X$ with $\|b-a\|=D$. Put $v=(b-a)/D$, $c=(a+b)/2$, and $z(x)=\langle v,x-c\rangle\in[-D/2,D/2]$. Set
\[
 A=\frac{\sigma}{4q^{1/p}}.
\]
The scalar laws $P_+$ and $P_-$ put masses $(3q/4,q/4,1-q)$ and $(q/4,3q/4,1-q)$, respectively, on $(A,-A,0)$. Their means are $\pm\Delta$. Moreover,
\begin{align*}
 \E|Y-\E Y|^p&\le2^p\E|Y|^p
 =2^p qA^p\notag\\
 &=2^{-p}\sigma^p\le\sigma^p,
 &\KL(P_-\|P_+)&=c_q.
\end{align*}
Return $g_t=Y_tv$ independently across rounds and independently of the query point. The null loss is $\ell_t(x)=\Delta z(x)$ and its oracle is $P_+$. In the alternative for window $I_j$, use $-\Delta z(x)$ and $P_-$ on $I_j$, retaining the null outside. All losses are fixed, and their selected gradients have norm $\Delta\le G$.

\textbf{A window with small null exploration.} Define $w(x)=z(x)/D+1/2\in[0,1]$. The null regret to $a$ is $D\Delta\sum_t w(x_t)$. Thus some deterministic window $I_j$ satisfies
\[
 \alpha_j:=\frac1n\sum_{t\in I_j}\E_0 w(x_t)
 \le\frac{B}{D\Delta T}=:\varepsilon<1.
\]
The index depends on the learner's distribution under the fixed null instance, not on a realized transcript. Choose this index for the alternative, and write
$\beta_j=n^{-1}\sum_{t\in I_j}\E_jw(x_t)$.

\textbf{Information and interval regret.} By the chain rule for relative entropy, the alternative transcript has divergence $nc_q$ from the null transcript: only the $n$ oracle draws in $I_j$ change their laws. Adaptive actions and the independent learner seed add no divergence. Adjoin an independent uniform round $U$ in $I_j$ and $V\sim\operatorname{Unif}[0,1]$. The event $\{V\le w(x_U)\}$ has probabilities $\alpha_j$ and $\beta_j$ under the two laws. Data processing and the binary-entropy bound give
\begin{align*}
 nc_q&\ge\operatorname{kl}(\beta_j\|\alpha_j)\notag\\
 &\ge\beta_j\log(1/\alpha_j)-\log2\notag\\
 &\ge\beta_j\log(1/\varepsilon)-\log2.
\end{align*}
If $\alpha_j=0$, absolute continuity forces $\beta_j=0$ and the conclusion is immediate. Otherwise rearrange the display. The alternative minimizer on $I_j$ is $b$, and its interval regret is exactly $D\Delta n(1-\beta_j)$. This proves \eqref{eq:testing-lower}, with the positive part justified by nonnegative roundwise regret. A zero null regret gives $\alpha_j=0$ on every window and hence the final assertion.
\end{proof}

\begin{corollary}[A necessary moment-dependent interval price]\label{cor:statistical-price}
Fix $C_0\ge1$, $p\in(1,2]$, and integers $T=Kn$ satisfying
\begin{equation}\label{eq:lower-regime}
 n\ge1+\log K,\qquad
 \rho\log K\ge4\bigl[1+\log(16C_0)\bigr].
\end{equation}
Consider any learner, even supplied with $G=\sigma>0$ and $p$, whose full-horizon expected regret to every fixed constant comparator is at most
$C_0D(G\sqrt T+\sigma T^{1/p})$ on every admissible instance. There is an admissible fixed instance and a deterministic length-$n$ interval, with a constant comparator on it, such that
\begin{equation}\label{eq:stat-price}
 \E\Reg_I(u)\ge\frac1{64}\sigma Dn^{1/p}(\log K)^\rho.
\end{equation}
In this regime the supplied-statistics upper bound \eqref{eq:known} for a static comparator is $O(\sigma Dn^{1/p}\chi_T(n)^\rho)$. Thus the noise power of the interval logarithm is sharp under the stated full-horizon requirement.
\end{corollary}
\begin{proof}
Set $b_0=\rho\log K$ and $q=b_0/(4n)\le1/8$. The null guarantee is at most $B=2C_0\sigma DT^{1/p}$. With \eqref{eq:testing-parameters},
\begin{align*}
 \log(D\Delta T/B)
 &=\rho\log(qT)-\log(16C_0)\notag\\
 &=b_0+\rho\log(b_0/4)-\log(16C_0)\notag\\
 &\ge3b_0/4.
\end{align*}
Here \eqref{eq:lower-regime} implies $b_0\ge8$ and $b_0\ge4\log(16C_0)$. In particular $B<D\Delta T$. Also
$nc_q+\log2\le b_0/4+\log2\le3b_0/8$.
Theorem~\ref{thm:testing} therefore gives at least $D\Delta n/2$, which equals
\[
 \frac{\sigma D}{16}\,n^{1/p}(\log K)^\rho(\rho/4)^\rho
 \ge\frac{\sigma D}{64}\,n^{1/p}(\log K)^\rho.
\]
The last inequality uses $\rho^\rho\ge1/2$ and $4^{-\rho}\ge1/2$ for $0<\rho\le1/2$. Finally $\chi_T(n)=1+\log K\le n$, so
$\sqrt{n\chi_T(n)}\le n^{1/p}\chi_T(n)^\rho$. The constant-comparator upper bound in \eqref{eq:known}, with $G=\sigma$, is therefore of the stated order. Its scale is at most a constant times the linear ceiling in this regime.
\end{proof}

The full-horizon requirement and the regime \eqref{eq:lower-regime} are part of the lower bound. The general tradeoff \eqref{eq:testing-lower} records the dependence on the allowed null regret, including the behavior as $p$ approaches one. At $p=2$ the logarithmic noise power is $1/2$; for $p<2$ it is $\rho<1/2$. The gap between the parameter-free $L^2$ upper cost and the supplied-statistics $\chi_T(n)$ cost therefore identifies calibration as a further question beyond the established statistical obstruction.

\section{Conclusions}\label{sec:conclusion}
One learner attains heavy-tailed dynamic regret on every fixed interval with an additive $\log^2 T$ adaptation cost and universal moment constants. Expected calibration and a local relative-entropy comparison support predictable sleeping experts, static comparators, and deterministic partitions. A shared nonuniform prior separates window location, duration, and restart complexity. With supplied statistics, the cost reduces to $1+\log(T/n)$; a regret--testing argument establishes its logarithmic noise power under a full-horizon optimality requirement. %Sharper calibration would bring the parameter-free guarantee closer to this statistical benchmark.

%\section*{AI Disclosure}
%Generative AI tools assisted with document editing.

\appendices

\section{Preliminaries and the base comparison}\label{app:prelim}
\subsection{Dyadic decomposition and moment conversion}
\begin{proof}[Proof of \eqref{eq:dyadic-power}]
Select the maximal untruncated dyadic intervals contained in $I$. Singletons guarantee coverage and nesting guarantees disjointness. A selected interval has a parent crossing an endpoint of $I$; at each scale at most one such selected child is associated with each endpoint. There are consequently at most two selected intervals of each size. Put $m=\lfloor\log_2|I|\rfloor$. For $\gamma\ge1/2$,
\begin{align*}
 \sum_J|J|^\gamma
 &\le2\sum_{k=0}^m2^{k\gamma}
 \le\frac{2}{1-2^{-\gamma}}|I|^\gamma\notag\\
 &\le(4+2\sqrt2)|I|^\gamma
 \le7|I|^\gamma.
\end{align*}
All selected windows belong to the labeled pool. Using untruncated windows in this decomposition avoids ambiguity between labels with coinciding truncated supports.
\end{proof}

\begin{lemma}[Predictably selected moments]\label{lem:moment}
Let $\beta_t\in\{0,1\}$ be predictable and $\sum_{t=1}^T\beta_t\le n$ almost surely, for deterministic $n\ge1$. Under \eqref{eq:model},
\begin{align}\label{eq:moment}
 \E\sqrt{\sum_t\beta_t\|g_t\|^2}
 &\le G\sqrt n+\sigma n^{1/p},\notag\\
 \E\max_{t:\beta_t=1}\|g_t\|
 &\le G+\sigma n^{1/p}.
\end{align}
The maximum of an empty set is zero. In particular, on any fixed set $S$ one may take $n=|S|$.
\end{lemma}
\begin{proof}
The product-space triangle inequality and monotonicity of finite-dimensional scalar norms give, pathwise,
\[
 \sqrt{\sum_t\beta_t\|g_t\|^2}
 \le G\sqrt n+\left(\sum_t\beta_t\|\epsilon_t\|^p\right)^{1/p}.
\]
Predictability and conditional moments imply
$\E\sum_t\beta_t\|\epsilon_t\|^p\le\sigma^p\E\sum_t\beta_t\le\sigma^pn$.
Jensen's inequality proves the first claim. The maximum is bounded by
$G+(\sum_t\beta_t\|\epsilon_t\|^p)^{1/p}$, proving the second. For a fixed set this is the marginal-moment conversion used by \citet{liu2026,aggarwal2026}. Predictable selection is the point at which the present general statement uses conditional moments.
\end{proof}

\subsection{Proof of Proposition~\ref{prop:base}}
For an actual restart block $B=[s,e]$ or its prefix, the identical deterministic inequality in \citet[Lemma 4.3]{aggarwal2026} states
\begin{align}\label{eq:inherited-expert}
 &\sum_{t\in B}\langle g_t,x_t^{(j)}-u_t\rangle\notag\\
 &\quad\le\sqrt2\left[D\sqrt{V_B}
 +\sum_{t=s+1}^e\|u_t-u_{t-1}\|
 \sqrt{V_{B,<t}}\right]\notag\\
 &\quad\le\sqrt2(D+P_B)\sqrt{V_B},
\end{align}
where $V_B=\sum_{t\in B}\|g_t\|^2$ and $V_{B,<t}=\sum_{r=s}^{t-1}\|g_r\|^2$. It applies to arbitrary realized gradient vectors, so unbiasedness at the virtual trajectory is unnecessary. We refer to that lemma for the identical pathwise proof. The calculation below adapts its application to a sleeping window.

Fix $J\in\DD_k$, $n=|J|$, and $\Lambda=\Lambda_J$. The selected restart length
$\ell=2^{\lfloor\log_2(n/\Lambda)\rfloor}$ has $r=\ell\Lambda/n\in(1/2,1]$. Since $\ell\le n\le2^k$, its dyadic left endpoint aligns with $J$. Thus the restriction to $J$ has $K=\lceil n/\ell\rceil$ fresh blocks or prefixes, each of length at most $\ell$. Truncation at $T$ preserves alignment, and $\sum_BP_B\le P_J$. Applying \eqref{eq:inherited-expert} and Lemma~\ref{lem:moment} separately on these deterministic blocks gives
\begin{align*}
 \E\sum_{t\in J}\langle g_t,x_t^{(j^*)}-u_t\rangle
 &\le\sqrt2(G\sqrt\ell+\sigma\ell^{1/p})(DK+P_J)\\
 &\le\sqrt2D(G\sqrt\ell+\sigma\ell^{1/p})(n/\ell+\Lambda).
\end{align*}
The coefficients relative to the two terms in \eqref{eq:base-bound} are
$\sqrt2(1+r^{-1})r^{1/2}$ and $\sqrt2(1+r^{-1})r^{1/p}$. Since $r^{1/p}\le\sqrt r$ and $\sqrt r+1/\sqrt r\le3/\sqrt2$ on $[1/2,1]$, both are at most three. This proves Proposition~\ref{prop:base} and checks the additional sleeping-window alignment explicitly.

\subsection{Exact-gradient comparison in Table~\ref{tab:comparison}}\label{app:exact-comparison}
We adapt the normalization of \citet[Theorem 5]{zhang2020adaptive}; its regret theorem itself is used by reference. Suppose $G>0$ and exact selected subgradients are available. Fix $x_0\in\X$, work on the translated and rescaled domain $\Omega=(\X-x_0)/D$, and run AOA on the affine surrogate
\[
 \widehat\ell_t(z)=\frac12+\frac{\langle g_t,z\rangle}{2G},\qquad z\in\Omega.
\]
This domain contains zero and has diameter one. Since $\|z\|\le1$ and $\|g_t\|\le G$, the surrogate lies in $[0,1]$ and its gradient norm is at most $1/2$. One exact query at the played $x_t=x_0+Dz_t$ determines the entire surrogate for all AOA subroutines. Its pathwise theorem applies to this realized sequence of affine losses.

The rescaled comparator has path length $P_I/D$. Extend it constantly beyond the endpoint when using the source's endpoint convention. In its Theorem 5, the logarithmic quantity $c'(s)$ is $O(L)$ for $s\le T$, and $\log(k_I+1)=O(\sqrt L)$ because $P_I/D\le T-1$. Substituting gradient bound $1/2$ and diameter one therefore gives
\[
 \sum_{t\in I}[\widehat\ell_t(z_t)-\widehat\ell_t((u_t-x_0)/D)]
 \le C\sqrt n\,[\sqrt L+\sqrt{P_I/D}+1].
\]
For example, the latter logarithmic estimate follows from $k_I+1\le2+\tfrac12\log_2(1+4T/7)$ and $\log x\le\sqrt x$ for $x\ge1$. Multiplying by $2GD$ and applying the subgradient inequality yields
\[
 \Reg_I(u)\le C'GD\sqrt{n(\Lambda_I+L)}
 =C'\mathcal B_I(\Lambda_I+L)\big|_{\sigma=0}.
\]
For $G=0$ the regret is zero. This supplies the common scaling and the one-gradient surrogate interpretation used in the table; it is an application of the cited exact-feedback result.

\section{Score bounds and allocation identities}\label{app:score}
\subsection{The simultaneous tangent and logarithmic bounds}
For $F(z)=1+\clip(z,-1/2,1/2)$, the tangent excess equals $(-1/2-z)_+$ on $z<-1/2$ and is zero otherwise. On its support it is at most $|z|\le2z^2$. For $|z|\le1/2$, integration of $t/(1+t)$ gives
$0\le z-\log(1+z)\le z^2$. On $z\ge1/2$, the positive logarithmic defect is at most $z\le2z^2$. On $z\le-1/2$, it is $(z+\log2)_+$, bounded by $|z|$ and by $2z^2$. This proves the first bound in \eqref{eq:defect-bound} for both defects.

If $|z|\le1$, use
$q(a+z)\le2(a+z)^2\le4a^2+4|z|^p$.
If $|z|>1$, use
$q(a+z)\le|a|+|z|\le a^2+1/4+|z|\le4a^2+4|z|^p$.
Now write
\begin{equation}\label{eq:noise-decomp}
 r_{t,i}=\mu_{t,i}+\xi_{t,i},\quad
 \mu_{t,i}=\langle v_t,x_t-y_{t,i}\rangle,\quad
 \xi_{t,i}=\langle\epsilon_t,x_t-y_{t,i}\rangle.
\end{equation}
Feasibility gives $|\mu_{t,i}|\le H$, while conditional centering and the conditional $p$th moment give
$\E[\xi_{t,i}\mid\F_{t-1}]=0$ and
$\E[|\xi_{t,i}|^p\mid\F_{t-1}]\le S^p$.
Substitution into \eqref{eq:defect-bound} proves \eqref{eq:h} for every predictable rate.

The score is related to bounded multiplicative updates, but its expectation need not be at most one for a centered input. For example, $Z=-2$ with probability $1/9$ and $Z=1/4$ otherwise has $\E Z=0$ and $\E F(Z)=7/6$. The calibration responds to this positive drift rather than assuming symmetry or treating truncation as unbiased.

\subsection{Complete allocation identities}
Suppress the time index. All $u_{i,l},A_l,a_l,\eta_l$ below are pre-round quantities. Prediction \eqref{eq:prediction} gives
\[
 \sum_{i,l\text{ awake}}u_{i,l}z_{i,l}
 =\left\langle g,x\sum_{i,l}u_{i,l}\eta_l-
 \sum_{i,l}u_{i,l}\eta_l y_i\right\rangle=0.
\]
Sleeping wealths do not change, so the potential ratio is
\begin{align*}
 \frac{\Phi^+}{\Phi^-}
 &=1+\sum_{i,l\text{ awake}}u_{i,l}[F(z_{i,l})-1-2a_l]\\
 &=1+C-\Gamma,\notag\\
 C&=\sum_{i,l\text{ awake}}
 u_{i,l}[F(z_{i,l})-1-z_{i,l}],\notag\\
 \Gamma&=\sum_l2a_lA_l.
\end{align*}
Thus this $C$ agrees with the quantity computed from the potential in \eqref{eq:allocation}. Since $F\in[1/2,3/2]$ and all divisors are at least one, $0<\Phi^+/\Phi^-\le3/2$. Also
\[
 0\le\Gamma=\sum_{i,l\text{ awake}}\frac{q_{i,l}W(i,l)}{\Phi^-}
 \frac{2a_l}{1+2a_l}\le1.
\]
It follows that $C\le3/2$, while $C_+\le d$ because $d$ sums the positive parts of the terms defining $C$. Therefore $d=0$ implies $C_+=0$, and the zero convention for allocation is valid. Further,
$0\le Z_l\le d_l$ and $\sum_lZ_l=C_+\le3/2<2$.
Taking conditional expectations gives $\E Z_l\le A_lh(\eta_l)$ by \eqref{eq:h}. Finally,
\[
 \log(\Phi^+/\Phi^-)\le C-\Gamma\le C_+-\Gamma
 =\sum_l(Z_l-2a_lA_l).
\]
These are all assertions in \eqref{eq:response}. The allocation can depend on the current gradient because $Z_l\le d_l$ is a pathwise domination by quantities whose coefficients are predictable.

\section{Proof of the general sleeping-master theorem}\label{app:master}
Fix a deterministic distribution $\nu$ as in Theorem~\ref{thm:master}, write $n=n_\nu$, and define
\begin{align}\label{eq:proof-constants}
 m_0&=9,\notag\\
 K_\nu&=1+\KL(\nu\|\pi)+\log(Q+1)\notag\\
 &\quad +(Q+1)(B_T+1)(\kappa+2),\notag\\
 \mathcal A_\nu&=m_0K_\nu.
\end{align}
Then $K_\nu\le C[1+\KL(\nu\|\pi)+L^2]$. Extend sleeping predictions by $x_0$ and put
\begin{align}\label{eq:weighted-comparison}
 w_t&=\sum_i\nu_i\alpha_{t,i},\notag\\
 a_t&=\sum_i\nu_i\alpha_{t,i}(x_t-y_{t,i}),
 &\|a_t\|&\le Dw_t,\notag\\
 \bar r_t&=\langle g_t,a_t\rangle,
 &\bar\mu_t&=\langle v_t,a_t\rangle,\notag\\
 \bar\xi_t&=\langle\epsilon_t,a_t\rangle.
\end{align}
All coefficients are predictable; $0\le w_t\le1$ and $\sum_t w_t\le n$ almost surely. In particular,
\begin{align}\label{eq:weighted-moments}
 |\bar\mu_t|&\le Hw_t,\notag\\
 \E[\bar\xi_t\mid\F_{t-1}]&=0,\notag\\
 \E[|\bar\xi_t|^p\mid\F_{t-1}]
 &\le S^pw_t^p\le S^pw_t.
\end{align}
This gives the sure upper bound $Hn$ on the true comparison.

If $G=0$, the true comparison is zero. If $\sigma>GT$, then $Hn\le Sn^{1/p}$, since $n^\rho\le T$. If $\mathcal A_\nu\ge n$, then $Hn\le H\sqrt{n\mathcal A_\nu}$. Each of these cases implies \eqref{eq:master-bound}. For the remaining proof suppose
\begin{equation}\label{eq:nontrivial}
 G>0,\qquad \sigma\le GT,\qquad 1\le\mathcal A_\nu<n\le T.
\end{equation}

\subsection{The signed comparison at small scales}
Let $b_t$ be the pre-round scale and $M_t=\max_{r\le t}\|g_r\|$, with $M_0=0$. Induction from \eqref{eq:trigger} gives
\begin{equation}\label{eq:scale-invariant}
 b_t\ge M_{t-1}/T^3.
\end{equation}
For a nontrigger observation, $\|g_t\|\le T^3b_t$. At a trigger, the new value $2\|g_t\|/T^2$ exceeds the previous scale and $\|g_t\|/T^3$. These two facts prove the induction, including the initial zero scale.

If $b_t<G/T^4$, then $M_{t-1}<G/T$. Once such a round observes $\|g_t\|\ge G/T$, every subsequent round has scale at least $G/T^4$. Thus, pathwise,
\[
 \sum_t w_t\|g_t\|\ind\{b_t<G/T^4\}
 \le Gn/T+\max_t w_t\|g_t\|.
\]
The weighted maximum obeys
\[
 \max_t w_t\|g_t\|
 \le G+\left(\sum_t w_t\|\epsilon_t\|^p\right)^{1/p}.
\]
Indeed, $w_t^p\le w_t$ and $\|v_t\|\le G$. By predictability,
$\E\sum_t w_t\|\epsilon_t\|^p\le\sigma^p\E\sum_t w_t\le\sigma^pn$.
Jensen's inequality gives $\E\max_t w_t\|g_t\|\le G+\sigma n^{1/p}$.
Both the weighted comparison and the small-scale test are predictable before the oracle draw. Conditional unbiasedness and $\|a_t\|\le Dw_t$ consequently imply
\begin{align}\label{eq:smallscale}
 \E\sum_t\bar\mu_t\ind\{b_t<G/T^4\}
 &=\E\sum_t\bar r_t\ind\{b_t<G/T^4\}\notag\\
 &\le2H+Sn^{1/p}.
\end{align}
The signed comparison is bounded before introducing the exceptional event used below. This order retains any negative contribution on its complement.

\subsection{A constant number of relevant phases}
A positive-scale phase begins just after a trigger observation and ends with the next trigger, when one occurs. Its scale is known before its first observation. Consider phases whose scale satisfies $b\ge G/T^4$. Each scale increase is by a factor greater than $2T$, and every positive scale is $2\|g_t\|/T^2$ for an earlier trigger observation.

Lemma~\ref{lem:moment} and \eqref{eq:nontrivial} give
\[
 \E M_T\le G+\sigma T^{1/p}\le2GT^2.
\]
Define
\begin{equation}\label{eq:good-event}
 \EE=\{M_T\le GT^5\},\qquad \Pr(\EE^c)\le2T^{-3}.
\end{equation}
On $\EE$, every phase under consideration has scale in $[G/T^4,2GT^3]$. Growth by more than $2T$ permits at most $m_0=9$ such phases. On the whole probability space, geometric growth also gives
\begin{equation}\label{eq:scale-sum}
 \sum_{\text{phases with }b\ge G/T^4}b\le4M_T/T^2.
\end{equation}
Select the first $m_0$ phases with $b\ge G/T^4$, even on trajectories that contain more. Membership in a selected phase is predictable. On $\EE$, these selected phases include every round outside the small-scale comparison.

\subsection{Predictable rate selection and expected penalties}
Fix the deterministic proof target
\begin{equation}\label{eq:target}
 \eta_* =\min\left\{\sqrt{\frac{\mathcal A_\nu}{400H^2n}},\quad
 \left(\frac{\mathcal A_\nu}{400S^pn}\right)^{1/p}\right\},
\end{equation}
with the noise entry infinite for $S=0$. Both entries are at least $1/(400HT^2)$ under \eqref{eq:nontrivial}. For a selected phase, its smallest grid rate satisfies
\[
 \eta_Q\le\frac{T^4}{8H}(2T)^{-8}
 =\frac{1}{2048HT^4}\le\frac{1}{400HT^2}\le\eta_*.
\]
Choose the largest grid rate at most $\eta_*$; when $\eta_0<\eta_*$ choose $\eta_0$. The choice is measurable at the phase start, and
\begin{equation}\label{eq:grid}
 0<\eta\le\eta_*,\qquad 1/\eta\le2/\eta_*+8Db.
\end{equation}

For selected phase $h$, apply Lemma~\ref{lem:overshoot} to its chosen rate and put
\[
 U_h=(a_{\max,h}-\delta)_+/h(\eta_h)
\]
when $h(\eta_h)>0$, with value zero if $h(\eta_h)=0$ or the phase never begins. Conditional on an existing phase's start, $\E U_h\le4$. Therefore
\begin{equation}\label{eq:overshoot-sum}
 R_*:=\sum_{h=1}^{m_0}U_h,\qquad \E R_*\le36,
 \qquad a_{t,l_h}\le\delta+R_*h(\eta_h)
\end{equation}
on every selected phase. The bound applies to the predictably selected rate, and the initial penalty floor remains separate from its random excess.

\subsection{Relative-entropy extraction on a phase}
Fix a selected phase and its rate index $l$. Every wealth starts at one at the phase reset, and $W(i,l)$ is updated on exactly the awake rounds of record $i$. Write $W_i$ for its value at the phase endpoint before a reset. The standard relative-entropy variational inequality follows directly from Jensen's inequality:
\begin{align}\label{eq:variational-full}
 \sum_i\nu_i\log W_i
 &=\KL(\nu\|\pi)
   +\sum_{i:\nu_i>0}\nu_i\log\frac{\pi_iW_i}{\nu_i}\notag\\
 &\le\KL(\nu\|\pi)+\log\sum_{i:\nu_i>0}\pi_iW_i\notag\\
 &\le\KL(\nu\|\pi)+\log(Q+1)+\log\Phi.
\end{align}
The final step uses $q_{i,l}=\pi_i/(Q+1)$ and positivity of all wealths. This is the usual mixture-comparison step in prior-sensitive expert aggregation \citep{luo2015}; the following calculation incorporates the finite-moment calibration and stopping phases in full.

Let $n_h=\sum_{t\text{ in phase}}w_t$ be the realized weighted activity of this phase. Since the chosen $\eta$ is constant over it, the wealth recursion and $\eta r\le\log F(\eta r)+\defect(\eta r)$ imply
\begin{align}
 \sum_{t\text{ in phase}}\bar r_t
 \le{}&\frac{K_\nu}{\eta}
 +\sum_{t\text{ in phase}}\sum_i\nu_i\alpha_{t,i}
       \frac{\defect(\eta r_{t,i})}{\eta}\notag\\
 &+8R_*n_hH^2\eta+8R_*n_hS^p\eta^{p-1}.\label{eq:phase-extract}
\end{align}
To verify the penalty term, \eqref{eq:overshoot-sum} gives
\[
 \log(1+2a_{t,l})\le2\delta+2R_*h(\eta)
 =2\delta+8R_*\eta^2H^2+8R_*\eta^pS^p.
\]
Its contribution in the averaged log wealth is multiplied by $w_t$. The added one in $K_\nu$ covers $2n_h\delta\le2T^{-5}<1$, and Lemma~\ref{lem:potential} bounds the remaining log-potential in \eqref{eq:variational-full}. All these inequalities hold pathwise at every possible phase endpoint. In particular, the argument keeps the comparison distribution inside logarithmic extraction, preserving the $\KL(\nu\|\pi)$ cost.

\subsection{Unconditional defects, penalties, and the martingale term}
Let $\theta_t$ indicate membership in one of the first $m_0$ selected phases. This is predictable. On those phases let $\eta_t$ be the chosen rate; set it to one elsewhere. By \eqref{eq:h}, predictability, and $\sum_t\theta_tw_t\le n$,
\begin{equation}\label{eq:defect-expectation}
 \E\sum_t\theta_t\sum_i\nu_i\alpha_{t,i}
 \frac{\defect(\eta_t r_{t,i})}{\eta_t}
 \le4nH^2\eta_*+4nS^p\eta_*^{p-1}.
\end{equation}
Likewise, $\sum_hn_h\le n$ pathwise, so \eqref{eq:overshoot-sum} gives
\begin{equation}\label{eq:penalty-expectation}
 \E\left[8R_*\sum_hn_h(H^2\eta_h+S^p\eta_h^{p-1})\right]
 \le288n(H^2\eta_*+S^p\eta_*^{p-1}).
\end{equation}
The almost-sure activity budget controls the dependence between $R_*$ and the random activity counts.

The single martingale sum $M=\sum_t\theta_t\bar\xi_t$ obeys
\begin{equation}\label{eq:martingale}
 \E|M|^p\le2S^pn,\qquad \E|M|\le2Sn^{1/p}.
\end{equation}
Here is a direct finite-moment proof. For real $a,b$ and $1<p\le2$,
\[
 |a+b|^p\le|a|^p+p\operatorname{sgn}(a)|a|^{p-1}b+2|b|^p.
\]
The derivative $p\operatorname{sgn}(x)|x|^{p-1}$ is $(p-1)$-H\"older with coefficient $p2^{2-p}$; integration of its difference proves the displayed inequality. Repeated conditional expectation cancels the linear term for the martingale increments. Using \eqref{eq:weighted-moments} yields
\[
 \E|M|^p\le2\sum_t\E|\theta_t\bar\xi_t|^p
 \le2S^p\E\sum_t\theta_tw_t\le2S^pn.
\]
H\"older's inequality justifies integrability of the linear term at each induction step. Jensen gives the second bound in \eqref{eq:martingale}. Aggregating the selected phases before taking the absolute value keeps a single finite-moment term.

\subsection{Completion of Theorem~\ref{thm:master}}
On $\EE$, selected phases contain all rounds with $b_t\ge G/T^4$. By \eqref{eq:grid} and \eqref{eq:scale-sum}, the sum of their first terms in \eqref{eq:phase-extract} is at most
\[
 \frac{2\mathcal A_\nu}{\eta_*}+32DK_\nu M_T/T^2.
\]
Subtracting $M$ converts their observed comparison to the true comparison, at cost at most $|M|$. All terms on this upper-bound side are nonnegative, so expectations restricted to $\EE$ are bounded by their unconditional expectations in \eqref{eq:defect-expectation}--\eqref{eq:martingale}. On $\EE^c$, the true comparison over rounds with $b_t\ge G/T^4$ is at most $Hn$. Combining these facts with the unconditional signed small-scale bound \eqref{eq:smallscale} gives, using $292\le400$,
\begin{align}
 \E\sum_t\bar\mu_t
 \le{}&2H+3Sn^{1/p}
 +\frac{2\mathcal A_\nu}{\eta_*}\notag\\
 &+400nH^2\eta_*+400nS^p\eta_*^{p-1}\notag\\
 &+32K_\nu\frac{H+ST^{1/p}}{T^2}
 +\frac{2Hn}{T^3}.\label{eq:preoptimized}
\end{align}

Put $A=\mathcal A_\nu$. The inverse of the minimum in \eqref{eq:target} is at most the sum of its two inverses. The three principal terms in \eqref{eq:preoptimized} are therefore bounded by
\[
 60H\sqrt{nA}+1200Sn^{1/p}A^\rho.
\]
For the mean term, the two coefficients are $2\sqrt{400}$ and $\sqrt{400}$. For the noise term their sum is $3\cdot400^{1/p}\le1200$. These constants are universal over $p$.

Under \eqref{eq:nontrivial}, $K_\nu\le T$ and $T^{1/p}\le T$. The residual scale terms are bounded by a universal multiple of $H+S$, and the two principal terms absorb all remaining terms because $n,A\ge1$. Finally $A\le C\mathcal K_\nu$ with a universal constant, and $\rho\le1/2$. This proves the nontrivial case of \eqref{eq:master-bound}. The sure bound and the initial case split complete the proof for every fixed $\nu$. Choosing a point mass gives the individual-record statement.

\section{The supplied-statistics algorithm and proof}\label{app:known}
\subsection{The common nonuniform prior and one rate per record}
Assume $G>0$ and that $G,\sigma,p$ are supplied. Use the same records, base trajectories, and prior $\pi$ as in \eqref{eq:pool}. Proposition~\ref{prop:prior} gives their normalization and selected-record cost. With these statistics supplied, each record uses one deterministic rate and its corresponding moment bound.

For record $i=(j,J)$ define deterministic quantities
\begin{align}
 n_i&=|J|, & k_i&=1+\log(1/\pi_i),\notag\\
 \eta_i&=\min\!\biggl\{\sqrt{\frac{k_i}{12H^2n_i}},\notag\\
 &\hspace{19mm}
 \left(\frac{k_i}{12S^pn_i}\right)^{1/p}\biggr\},\notag\\
 h_i&=4\eta_i^2H^2+4\eta_i^pS^p.
 \label{eq:known-rates}
\end{align}
where the noise entry is infinite for $S=0$. Initialize wealths $W_i=1$ and $\Phi=\sum_i\pi_iW_i=1$. On each round, for awake records use
\begin{align}\label{eq:known-update}
 u_{t,i}&=\frac{\pi_iW_i}{\Phi^-(1+2h_i)},\notag\\
 x_t&=\frac{\sum_{i\text{ awake}}u_{t,i}\eta_i y_{t,i}}
 {\sum_{i\text{ awake}}u_{t,i}\eta_i},\notag\\
 W_i^+&=W_i\frac{F(\eta_i r_{t,i})}{1+2h_i}.
\end{align}
All other wealths remain unchanged. The known penalties remain fixed throughout the run. Unborn weights are one and expired contributions are frozen. The base trajectories are still \eqref{eq:base-update}. Each record uses its prescribed support length; the learner maintains all records regardless of which interval is later evaluated.

There are at most $(N+1)(N+2)/2$ awake records and $N+1$ base trajectories. Thus the algorithm has the operation and storage counts in Theorem~\ref{thm:known}, with one shared oracle draw. All rates and weights are predictable.

\subsection{Expected logarithmic potential}
By the tangent bound and \eqref{eq:h},
\[
 \E[F(\eta_i r_{t,i})\mid\F_{t-1}]
 \le1+\eta_i\mu_{t,i}+h_i.
\]
The mixture in \eqref{eq:known-update} cancels $\sum_i u_{t,i}\eta_i\mu_{t,i}$. The complete potential therefore obeys
\begin{equation}\label{eq:known-supermartingale}
 \E[\Phi^+\mid\F_{t-1}]
 \le\Phi^-\left(1-\sum_{i\text{ awake}}u_{t,i}h_i\right)\le\Phi^-.
\end{equation}
It is a nonnegative supermartingale starting at one. Since the finite collection of $h_i$ is deterministic and finite, and $F\in[1/2,3/2]$, the logarithms of all wealths and of the potential are bounded in absolute value by finite deterministic constants over the fixed run. Consequently Jensen's inequality is justified at a fixed endpoint:
\begin{equation}\label{eq:known-log}
 \E\log\Phi_{\max J}\le\log\E\Phi_{\max J}\le0.
\end{equation}

For record $i=(j,J)$, wealth is one at birth and is updated only on $J$. Using $\eta_i r\le\log F(\eta_i r)+\defect(\eta_i r)$ and $\pi_iW_i\le\Phi$, we get
\begin{align*}
 \eta_i\sum_{t\in J}r_{t,i}
 &\le\log\Phi_{\max J}+\log(1/\pi_i)\notag\\
 &\quad+\sum_{t\in J}\defect(\eta_i r_{t,i})
 +n_i\log(1+2h_i).
\end{align*}
Take expectations, apply conditional unbiasedness, \eqref{eq:known-log}, the expected defect bound $h_i$, and $\log(1+2h_i)\le2h_i$. This yields
\begin{align*}
 \E\sum_{t\in J}\mu_{t,i}
 &\le\frac{k_i}{\eta_i}+12n_iH^2\eta_i
 +12n_iS^p\eta_i^{p-1}\notag\\
 &\le7H\sqrt{n_ik_i}+24Sn_i^{1/p}k_i^\rho.
\end{align*}
For the last step use the two balancing rates in \eqref{eq:known-rates}; the coefficients are at most $2\sqrt{12}<7$ and $2\cdot12^{1/p}\le24$. This expected-potential argument applies for every $\sigma/G\ge0$.

\subsection{Selected prior cost and a weighted cover}
For each untruncated piece $J$ in the maximal cover of $I$, Proposition~\ref{prop:prior} gives
\begin{equation}\label{eq:known-prior-cost}
 k_{(j^*,J)}\le c_T(|J|,\Lambda_J)
 \le C[\chi_T(|J|)+\Lambda_J].
\end{equation}
For the second inequality, use $\log(3+z)\le\log3+z/3$ for $z\ge0$, and $\log\lambda\le\lambda-1$ for $\lambda\ge1$. Thus the scale term is bounded by a constant times $\chi_T(|J|)$ and the restart term by a constant times $\Lambda_J$. Taking the power $\rho$ after this addition keeps the constants universal as $p$ approaches one.

\begin{lemma}[Weighted geometric cover]\label{lem:weighted-cover}
For the maximal dyadic decomposition of an interval of length $n$, every $\gamma\in[1/2,1]$ and $a\in[0,1/2]$ satisfy
\begin{align}
 \sum_J|J|^\gamma\chi_T(|J|)^a
 &\le32n^\gamma\chi_T(n)^a,\label{eq:weighted-cover}\\
 \sum_J|J|^\gamma c_T(|J|,\Lambda_J)^a
 &\le32n^\gamma c_T(n,\Lambda_I)^a.\label{eq:prior-weighted-cover}
\end{align}
\end{lemma}
\begin{proof}
Put $m=\lfloor\log_2n\rfloor$. At length $2^{m-r}$ there are at most two pieces, with $|J|/n\le2^{-r}$ and $\log(n/|J|)<(r+1)\log2$. For \eqref{eq:weighted-cover},
\[
 \chi_T(|J|)\le\chi_T(n)+(r+1)\log2.
\]
For \eqref{eq:prior-weighted-cover}, $\Lambda_J\le\Lambda_I$. Regarding the location-dependent part of $c_T$ as a function of $x=\log(T/n)$, its derivative is
\[
 1+\frac{2}{3\log2+x}<2\qquad(x\ge0).
\]
Monotonicity in $\lambda$ and integration of this derivative give
\begin{align*}
 c_T(|J|,\Lambda_J)
 &\le c_T(n,\Lambda_I)+2\log(n/|J|)\notag\\
 &\le c_T(n,\Lambda_I)+2(r+1)\log2.
\end{align*}
Both reference complexities are at least one. After dividing by the corresponding right-hand scale, the contribution of a piece is at most
$2^{-r/2}\sqrt{1+2(r+1)\log2}\le2^{-r/2}(r+2)$.
Summing at most two pieces of each length yields
\[
 2\sum_{r\ge0}2^{-r/2}(r+2)=16+10\sqrt2<32.
\]
This proves both inequalities with the stated universal constant.
\end{proof}
The geometric-cover argument is the usual reduction of interval comparisons to dyadic pieces \citep{daniely2015,jun2017}, with the prior-dependent weight retained here. Combining the individual-record bound, Proposition~\ref{prop:base}, and \eqref{eq:prior-weighted-cover} proves the refined bound \eqref{eq:known-refined}. To obtain the separated form in Theorem~\ref{thm:known}, first use \eqref{eq:known-prior-cost} on each piece, and then \eqref{eq:weighted-cover} with $(\gamma,a)=(1/2,1/2)$ and $(1/p,\rho)$. The movement sums are exactly those in the proof of Theorem~\ref{thm:main}. Finally $x^a+y^a\le2(x+y)^a$ gives \eqref{eq:known}.

\subsection{Proof of Corollary~\ref{cor:known-static}}
Use the root record $(N,[T])$. Its trajectory has one uninterrupted AdaGrad block, its prior mass is at least $1/4$, and a constant comparator has zero internal movement. The identical base inequality \eqref{eq:inherited-expert} and the moment bound give base regret at most $\sqrt2(H\sqrt T+ST^{1/p})$. The record comparison above has $k_i\le1+\log4$, so the total is at most
\[
 \bigl[\sqrt2+7\sqrt{1+\log4}\bigr]H\sqrt T
 +\bigl[\sqrt2+24(1+\log4)^\rho\bigr]ST^{1/p}.
\]
Both coefficients are at most $40$ for $0<\rho\le1/2$. Linearization and the regret ceiling complete the proof. The same learner and oracle observations support all the interval guarantees in Theorem~\ref{thm:known}.


\begin{thebibliography}{19}
\providecommand{\natexlab}[1]{#1}

\bibitem[Aggarwal(2026)]{aggarwal2026}
V. Aggarwal.
\newblock Parameter-free dynamic regret under heavy-tailed noise.
\newblock \emph{arXiv preprint arXiv:2607.27073}, 2026.

\bibitem[Cutkosky(2020)]{cutkosky2020}
A. Cutkosky.
\newblock Parameter-free, dynamic, and strongly-adaptive online learning.
\newblock In \emph{Proceedings of the International Conference on Machine Learning}, 2020.

\bibitem[Daniely et~al.(2015)Daniely, Gonen, and Shalev-Shwartz]{daniely2015}
A. Daniely, A. Gonen, and S. Shalev-Shwartz.
\newblock Strongly adaptive online learning.
\newblock In \emph{Proceedings of the International Conference on Machine Learning}, 2015.

\bibitem[de~Rooij et~al.(2014)de~Rooij, van Erven, Gr\"unwald, and Koolen]{derooij2014}
S. de~Rooij, T. van Erven, P. D. Gr\"unwald, and W. M. Koolen.
\newblock Follow the leader if you can, {Hedge} if you must.
\newblock \emph{Journal of Machine Learning Research}, 15:1281--1316, 2014.

\bibitem[Duchi et~al.(2011)Duchi, Hazan, and Singer]{duchi2011}
J. Duchi, E. Hazan, and Y. Singer.
\newblock Adaptive subgradient methods for online learning and stochastic optimization.
\newblock \emph{Journal of Machine Learning Research}, 12:2121--2159, 2011.

\bibitem[Gaillard et~al.(2014)Gaillard, Stoltz, and van Erven]{gaillard2014}
P. Gaillard, G. Stoltz, and T. van Erven.
\newblock A second-order bound with excess losses.
\newblock In \emph{Proceedings of the Conference on Learning Theory}, 2014.

\bibitem[Hazan and Seshadhri(2009)]{hazan2009}
E. Hazan and C. Seshadhri.
\newblock Efficient learning algorithms for changing environments.
\newblock In \emph{Proceedings of the International Conference on Machine Learning}, 2009.

\bibitem[Howard et~al.(2020)Howard, Ramdas, McAuliffe, and Sekhon]{howard2020}
S. R. Howard, A. Ramdas, J. McAuliffe, and J. Sekhon.
\newblock Time-uniform {Chernoff} bounds via nonnegative supermartingales.
\newblock \emph{Probability Surveys}, 17:257--317, 2020.

\bibitem[Jun et~al.(2017)Jun, Orabona, Wright, and Willett]{jun2017}
K.-S. Jun, F. Orabona, S. Wright, and R. Willett.
\newblock Improved strongly adaptive online learning using coin betting.
\newblock In \emph{Proceedings of the International Conference on Artificial Intelligence and Statistics}, 2017.

\bibitem[Liu(2026)]{liu2026}
Z. Liu.
\newblock Online convex optimization with heavy tails: Old algorithms, new regrets, and applications.
\newblock In \emph{Proceedings of the International Conference on Algorithmic Learning Theory}, 2026.

\bibitem[Luo and Schapire(2015)]{luo2015}
H. Luo and R. E. Schapire.
\newblock Achieving all with no parameters: {AdaNormalHedge}.
\newblock In \emph{Proceedings of the Conference on Learning Theory}, 2015.

\bibitem[Moulin et~al.(2025)Moulin, Esposito, and van der Hoeven]{moulin2025}
A. Moulin, E. Esposito, and D. van der Hoeven.
\newblock When lower-order terms dominate: Adaptive expert algorithms for heavy-tailed losses.
\newblock \emph{arXiv preprint arXiv:2506.01722}, 2025.

\bibitem[Orabona and P\'al(2018)]{orabona2018}
F. Orabona and D. P\'al.
\newblock Scale-free online learning.
\newblock \emph{Theoretical Computer Science}, 716:50--69, 2018.

\bibitem[Vural et~al.(2022)Vural, Yu, Balasubramanian, Volgushev, and Erdogdu]{vural2022}
N. M. Vural, L. Yu, K. Balasubramanian, S. Volgushev, and M. A. Erdogdu.
\newblock Mirror descent strikes again: Optimal stochastic convex optimization under infinite noise variance.
\newblock In \emph{Proceedings of the Conference on Learning Theory}, 2022.

\bibitem[Xie et~al.(2026)Xie, Wang, Zhao, and Zhou]{xie2026}
Y.-F. Xie, S. Wang, P. Zhao, and Z.-H. Zhou.
\newblock Online learning with gradient-variation interval regret.
\newblock \emph{arXiv preprint arXiv:2606.03831}, 2026.

\bibitem[Zhang et~al.(2018)Zhang, Lu, and Zhou]{zhang2018}
L. Zhang, S. Lu, and Z.-H. Zhou.
\newblock Adaptive online learning in dynamic environments.
\newblock In \emph{Advances in Neural Information Processing Systems}, 2018.

\bibitem[Zhang et~al.(2020)Zhang, Lu, and Yang]{zhang2020adaptive}
L. Zhang, S. Lu, and T. Yang.
\newblock Minimizing dynamic regret and adaptive regret simultaneously.
\newblock In \emph{Proceedings of the International Conference on Artificial Intelligence and Statistics}, 2020.

\bibitem[Zhang and Cutkosky(2022)]{zhangcutkosky2022}
J. Zhang and A. Cutkosky.
\newblock Parameter-free regret in high probability with heavy tails.
\newblock In \emph{Advances in Neural Information Processing Systems}, 2022.

\bibitem[Zinkevich(2003)]{zinkevich2003}
M. Zinkevich.
\newblock Online convex programming and generalized infinitesimal gradient ascent.
\newblock In \emph{Proceedings of the International Conference on Machine Learning}, 2003.

\end{thebibliography}
\end{document}